\documentclass[11pt,journal,compsoc,twocolumn]{IEEEtran}
\ifCLASSOPTIONcompsoc
  \usepackage[nocompress]{cite}
\else
  normal IEEE
 \usepackage{cite}
\fi
\usepackage{graphicx}
\usepackage{multirow}
\usepackage{amsmath}
\usepackage{amsfonts}
\usepackage{amssymb}
\usepackage{graphicx}
\usepackage{caption}
\usepackage{subfig}
\usepackage{tabularx}
\usepackage{amsthm}
\usepackage[ruled]{algorithm}
\usepackage[compatible]{algpseudocode}
\usepackage{verbatim}
\usepackage{enumitem}
\usepackage{mathtools}
\usepackage{xcolor}
\usepackage{algpseudocode}
\usepackage[textwidth=8em,textsize=small]{todonotes}
\usepackage{flushend}

\makeatletter
\let\OldStatex\Statex
\renewcommand{\Statex}[1][3]{%
  \setlength\@tempdima{\algorithmicindent}%
  \OldStatex\hskip\dimexpr#1\@tempdima\relax}
\makeatother

\newtheorem{prop}{Proposition}[section]
\newtheorem{corr}{Corollary}[section]
\newtheorem{lemma}{Lemma}[section]

\algnewcommand\algorithmicswitch{\textbf{switch}}
\algnewcommand\algorithmiccase{\textbf{case}}
\algnewcommand\algorithmicassert{\texttt{assert}}
\algnewcommand\Assert[1]{\State \algorithmicassert(#1)}%
\algdef{SE}[SWITCH]{Switch}{EndSwitch}[1]{\algorithmicswitch\ #1\ \algorithmicdo}{\algorithmicend\ \algorithmicswitch}%
\algdef{SE}[CASE]{Case}{EndCase}[1]{\algorithmiccase\ #1}{\algorithmicend\ \algorithmiccase}%
\algtext*{EndSwitch}%
\algtext*{EndCase}%

\begin{document}

\title{A stochastic approach to handle knapsack problems in the creation of ensembles}

\author{Andr\'as~Hajdu$^*$,~\IEEEmembership{Senior Member,~IEEE,}
     Gy\"orgy~Terdik,
      Attila Tiba,
      and~Henrietta~Tom\'an~%
\IEEEcompsocitemizethanks{\IEEEcompsocthanksitem A. Hajdu$^*$, Gy. Terdik, A. Tiba, and H. Tom\'an are with the Faculty of Informatics, University of Debrecen, 4002 Debrecen, POB 400, Hungary.\protect\\
E-mail: $\{$hajdu.andras,~terdik.gyorgy,~tiba.attila,~toman.henrietta$\}$\protect\\@inf.unideb.hu.
}
\thanks{{*} Corresponding author.}
}

\IEEEtitleabstractindextext{
\begin{abstract}
Ensemble-based methods are highly popular approaches that increase the accuracy of a decision by aggregating the opinions of individual voters. The common point is to maximize accuracy; however, a natural limitation occurs if incremental costs are also assigned to the individual voters. Consequently, we investigate creating ensembles under an additional constraint on the total cost of the members. This task can be formulated as a knapsack problem, where the energy is the ensemble accuracy formed by some aggregation rules. However, the generally applied aggregation rules lead to a nonseparable energy function, which takes the common solution tools -- such as dynamic programming -- out of action. We introduce a novel stochastic approach that considers the energy as the joint probability function of the member accuracies. This type of knowledge can be efficiently incorporated in a stochastic search process as a stopping rule, since we have the information on the expected accuracy or, alternatively, the probability of finding more accurate ensembles. Experimental analyses of the created ensembles of pattern classifiers and object detectors confirm the efficiency of our approach. Moreover, we propose a novel stochastic search strategy that better fits the energy, compared with general approaches such as simulated annealing.   

\end{abstract}

} 

\maketitle

\IEEEpeerreviewmaketitle

\ifCLASSOPTIONcompsoc
\IEEEraisesectionheading{\section{Introduction}\label{sec:introduction}}
\else
\section{Introduction}
\label{sec:introduction}
\fi

\IEEEPARstart{E}{}nsemble-based
systems are rather popular in several application fields and are employed to increase the decision accuracy of individual approaches. We 
also encounter such approaches for pattern recognition purposes \cite%
{numbl}, using models based on, e.g., neural networks \cite{nn1,nn2}, decision
trees \cite{dt1} or other principles \cite{othmod1,othmod2,othmod3}. In the most recent results, we can recognize this approach in the design of state-of-the-art convolutional neural networks (such as GoogLeNet, incorporating the Inception module \cite{7298594}) or the direct combination of them \cite{HarangiBH18}.
In our practice, we also routinely consider ensemble-based approaches to aggregate the outputs of pattern classifiers \cite{ANTAL201420} or detector algorithms \cite{6177649}, usually by some majority voting-based rule. During these efforts, we have also faced perhaps the most negative property of creating ensembles, that is, the increasing demand on resources. This type of cost may occur as the execution/training time and the working hours needed to create the ensemble components, etc., according to the characteristics of the given problem. Thus, in addition to the primary aim of the composition of the most accurate ensemble, a natural constraint emerges as a cost limitation for that endeavor. 

More formally, let us consider a pool ${\cal D}=\{{\cal D}_1,\dots,{\cal D}_n\}$ containing possible ensemble members, where each ${\cal D}_i$ $(i=1,\dots,n)$ is characterized by a pair $(p_i,t_i)$ describing its individual accuracy $p_i \in [0,1]$ and cost $t_i\in\mathbb R_{> 0}$. The individual accuracies 
are supposed to be known, e.g., by determining them on some test data and by an appropriate performance metric.
In this work, we will focus on the majority voting-based aggregation principle, where the possible ensemble members ${\cal D}_i$ $(i=1,\dots,n)$ are classifiers (see \cite{kb}). 
In \cite{Hajdu2013jspatialvoting}, we have dealt with the classic case in which the individual classifiers make true/false (binary) decisions. In this model, a classifier $D_i$ with accuracy $p_i$ is considered as a Bernoulli distributed random variable $\eta_i$, that is, $P(\eta_i= 1) =p_i$,  $P(\eta_i = 0) = 1- p_i$ $(i= 1, \ldots, n)$, where $\eta_i= 1$ means the correct classification by $D_i$.
In this case, we obtain that the accuracy of an ensemble ${\cal D}'=\{{\cal D}_{i_1},\dots,{\cal D}_{i_\ell}\}\subseteq {\cal D}$ of $|{\cal D}'|=\ell$ members can be calculated as

\begin{equation}
q_\ell({\cal L})=\sum\limits_{k=\left\lfloor \frac{\ell}{2}\right\rfloor +1}^{\ell}\left( {%
\underset{|{\cal I}|=k}{\sum\limits_{{\cal I}\subseteq {\cal L}}}}\prod\limits_{i\in
{\cal I}}p_{i}\prod\limits_{j\in {\cal L}\setminus {\cal I}}\left( 1-p_{j}\right) \right) ,
\label{accuracy}
\end{equation}
where ${\cal L}=\{i_1,\dots,i_\ell\}\subseteq {\cal N}=\{ 1,\dots,n\}$ is the index set of ${\cal D}'$. As an important practical issue,
notice that (\ref{accuracy}) is valid only
for independent members to calculate the ensemble accuracy. The dependency of the members can be discovered further by,
e.g., using different kinds of diversity measures \cite{Hajdu2013cdivmes}.

Regarding ensemble-based systems, the standard task is to devise the most accurate
ensemble from ${\cal D}$ for the given energy function. In this paper, we add a natural constraint of a bounded total cost to this optimization problem. That is, we have to maximize (\ref{accuracy}) under the cost constraint
\begin{equation}
\sum\limits_{i\in {\cal L}}t_i\leq T,  \label{qTcondition}
\end{equation}%
where the total allowed cost $T\in\mathbb R_{> 0}$ is a predefined constant. 
Consequently, we must focus on those subsets ${\cal L}\subseteq{\cal N}$ with cardinalities $\left\vert {\cal L}\right\vert
=\ell\in\{1,\dots,n\}$ for which (\ref{qTcondition}) is fulfilled. Let ${\cal L}_{0}$ denote that index set of cardinality $\left\vert {\cal L}_{0}\right\vert =\ell_{0}$, where the global maximum ensemble accuracy is reached.
The following lemma states that
one can reach ${\cal L}_{0}$ calculating $q_{\ell}\left(
{\cal L}\right)$ for odd values of $\ell$ only, which results in more efficient computation, since not all the possible subsets should be checked.

\begin{lemma}
\label{OddLemma}If
\begin{equation}
\max\left(  q_{\ell}\left( {\cal L}\right)  ~|~ {\cal L}\subseteq {\cal N}\right)  =q_{\ell_{0}}\left(
{\cal L}_{0}\right)  ,
\end{equation}
then $\ell_{0}$ is odd.
\end{lemma}

\begin{proof}
See Appendix \ref{Proof_Odd_Lemma} for the proof.
\end{proof}

The optimization task defined by the energy function (\ref {accuracy}) and the constraint (\ref{qTcondition}) can be interpreted as a typical knapsack problem \cite{linknap1}. Such problems are known to be NP-hard; however, if the energy function is linear and/or separable for the $p_i$-s, then a very efficient algorithmic solution can be given based on dynamic programming. However, if the energy lacks these properties, the currently available theoretical foundation is rather poor. As some specific examples, we were able to locate investigations of an exponential-type energy function \cite{KLASTORIN1990233}, and a remarkably restricted family of nonlinear and nonseparable ones \cite{nonlinknap3}. In \cite{KLASTORIN1990233}, an approach based on calculus was made by representing the energy function by its Taylor series. Unfortunately, it has been revealed that dynamic programming can be applied efficiently only to at most the quadratic member of the series; thus, the remaining higher-order members had to be omitted. This compulsion suggests a large error term if this technique is attempted to be considered generally. Thus, to the best of our knowledge, there is a lack of theoretical instructions/general recommendations to solve knapsack problems in the case of complex energy functions. 
As our energy (\ref{accuracy}) is also nonlinear and nonseparable, we were highly motivated to develop a well-founded framework for efficient ensemble creation. 

As a common empirical approach to find the optimal ensemble, the usefulness  $p_i/t_i$ ($i=1,\dots,n$) of the possible members are calculated. Then, as deterministic greedy methods, the ensemble is composed of forward/backward selection strategies 
(see, e.g., \cite{SensorConfiguration}). Since the deterministic methods are less efficient -- e.g., the greedy one is proven to have 50$\%$ accuracy for the simplest knapsack energy $\sum_{i=1}^n p_i$ -- popular stochastic search algorithms are considered instead, such as simulated annealing (SA). However, for the sake of completeness, we will start our theoretical investigation regarding the accuracy of the existing deterministic methods when a cost limitation is also applied.

As our main contribution, in this paper, we propose a novel
stochastic technique to solve knapsack problems with complex energy functions. Though the model is worked out in detail for 
(\ref {accuracy}) settled on the majority voting rule, it can also be applied to other energy functions. Our approach is based on the stochastic properties of the energy 
$q_\ell$ 
in (\ref {accuracy}), providing that we have some preliminary knowledge on where the distribution its parameters $p_i$ $(i=1,\dots,n)$ is coming from, with a special focus on \emph{beta} distributions that fit practical problems very well. In other words, we estimate the distribution of $q_\ell$ in terms of its mean and variance. These information can be efficiently incorporated as a stopping rule in stochastic search algorithms, as we demonstrate it for SA. The main idea here is to be able to stop building ensembles when we can expect that better ones can be found by low probability. 

As a further contribution, we introduce a novel stochastic search strategy, where the usefulness of the components are defined in a slightly more complex way to better fit the investigated energy; the stopping rule can be successfully applied in this approach, as well. 
Our empirical analyses also show that including the stochastic estimation as a stopping rule saves a large amount of search time to build accurate ensembles. Moreover, our novel search strategy is proven to be very competitive with SA. Our stochastic approach was first proposed in our former work \cite{7899637} with limited empirical evaluations; however, only heuristic results were achieved there without being able to take advantage of the theoretical model completed here. 

The rest of the paper is organized as follows. In section \ref{sec:selstrat}, we analyze the maximum accuracy of the common deterministic ensemble creator strategies in the case of limited total cost. Existing stochastic approaches are described in section \ref{sec:selstrat2} with some preliminary simulation results. Moreover, we introduce a novel stochastic search algorithm that determines the expected usefulness of possible members in a way that adapts to the characteristics of the energy function better than, e.g., SA. The stochastic estimation of the ensemble energy from the individual accuracies of the components is presented in section \ref{sec:stoc}; our code is available at https://codeocean.com/capsule/3944336. Our experimental analysis are enclosed in section \ref{sec:empanal}, including the investigation of the possible creation of ensembles from participating methods of Kaggle challenges and binary classification problems in UCI databases; our data is available at https://ieee-dataport.org/documents/binary-classifiers-outputs-ensemble-creation. We also present how the proposed model is expected to be generalized to multiclass classification tasks with a demonstrative example on our original motivating object detection problem. Finally, in section \ref{sec:disc}, we discuss several issues regarding our approach that can be tuned towards special application fields.

\section{Deterministic selection strategies}
\label{sec:selstrat}

In this section, we address deterministic selection strategies to build an
ensemble that has maximal system accuracy $q_{\ell_0}({\cal L}_0)$, applying the cost limitation. 
However, since we have $2^{n}$ different subsets of elements of a pool of cardinality $n$, this selection task is known to be NP-hard. To overcome
this issue, several selection approaches have been proposed.
The common point of these strategies is
that in general, they do not assume any knowledge on the proper determination
of the classification performance $q_\ell({\cal L})$; rather, they require only the ability to evaluate
it. Moreover, to the best of our knowledge, strategies that consider the capability of individual feature accuracies to be modeled by
drawing them from a probability distribution, as in our approach, have not yet been
proposed. 

Based on the above discussion, it seems to be natural to ascertain how the widely
applied selection strategies work in our
setup. The main difference in our case, in contrast to the general
recommendations, is that now we can properly formulate the performance evaluation 
using the exact functional knowledge of $q_\ell$. That is, we can
characterize the behavior of the strategy with a strict analysis instead of the
empirical tests generally applied.

We start our investigation with greedy selection approaches by discussing
them via the forward selection strategy. 
Here, the most accurate item is selected and put in a subset $S$ first. Then, from the remaining $n-1$ items, the component that maximizes the classification accuracy
of the extended ensemble is moved to $S$
.
This procedure is then iteratively repeated; however, if the performance
cannot be increased by adding a new component, then $S$
is not extended and the selection stops. 
The first issue we address is to determine the largest possible error this strategy can lead to in our scenario.

\begin{prop}
\label{prop:grfo} The simple greedy forward selection strategy to build an ensemble that applies the majority voting-based rule has a maximum error
rate $1/2$.
\end{prop}

\begin{proof}
For the proof, see Appendix \ref{Proof_prop:grfo}.
\end{proof}

As seen from the proof, the error rate of $1/2$ holds for the forward strategy independent of the time constraint.
As a quantitative example, let $p_1=0.510$ and $p_2=p_3=p_4=p_5=0.505$.
With this setup, where ${\cal{I}}_k=\{1,\dots,k\}$, we have $q_1({\cal{I}}_1)=p_1=0.5100$, $q_3({\cal{I}}_3)=0.5100$, and $%
q_5({\cal{I}}_5)=0.5112$, which shows that the greedy forward selection strategy is stuck at the single element ensemble, though a more accurate larger one could be found.

In addition to forward selection, its inverse variant, the backward selection strategy, is also popular. It puts all the components into an ensemble first, and in every selection step, leaves the
worst one out to gain maximum ensemble accuracy. 
As a major difference from the forward strategy, backward selection is efficient in our case if the time constraint is irrelevant. Namely, either the removal of the worst items will lead to an increase in $q_{\ell}$ defined in (\ref {accuracy}), or the selection can be stopped without the risk of missing a more accurate ensemble.
 However, if the time constraint applies, the same maximum error rate 
can be proved.

\begin{prop}
\label{prop:grba} The simple greedy backward selection strategy considering
the individual accuracy values to build an ensemble that applies the
majority voting-based rule has a maximum error rate of 1/2.
\end{prop}

\begin{proof}
Proposition \ref{prop:grba}. is proved in Appendix \ref{Proof_prop:grba}.
\end{proof}

Propositions \ref{prop:grfo}. and \ref{prop:grba}. have shown the worst case
scenarios for the forward and backward selection strategies. However, the
greedy approach was applied only regarding the accuracy values of the
members, and their execution times were omitted. To consider both the
accuracies and execution times of the algorithms in the ensemble pool $%
{\cal D}
=\{D_{1}=(p_{1},t_{1}),D_{2}=(p_{2},t_{2}),\dots,D_{n}=(p_{n},t_{n})\}$, we consider their usefulness in the selection
strategies, defined as 
\begin{equation}
u_{i}=p_i/t_i,~~i=1,\dots ,n,  \label{eq:usef}
\end{equation}%
which is a generally used definition to show the price-to-value ratio of an object.
The composition of
ensembles based on similar usefulness measures has also been efficient, e.g., in sensor networks \cite{SensorConfiguration}.

After the introduction of the usefulness (\ref{eq:usef}), the first
natural question to clarify is to investigate whether the validity of the error rate of the
deterministic greedy forward and backward selection strategies operating
with the usefulness measure holds. Through the following two statements, we will see that the $1/2$ error rates remain valid
for both greedy selection approaches.

\begin{corr}
\label{prop:grfou} Proposition \ref{prop:grfo} remains valid when the
forward feature selection strategy operates on the usefulness.
Namely, as a worst case scenario, let $t_{1}=t_{2}=\dots =t_{n}=T/n$ be the
execution times in the example of the proof of Proposition \ref{prop:grfo}
while keeping the same $p_{1},p_{2},\dots ,p_{n}$ values. Then, the
selection strategy operates completely in the same way on the $%
u_{i}=p_{i}/t_{i}$ values $(i=1,\dots,n)$ as on the $p_{i}$ ones, since the $t_{i}$
values are equal. That is, the error rate is $1/2$ in the same way.
\end{corr}

\begin{prop}
\label{prop:grbau} The simple greedy backward selection strategy considering
the individual usefulness (\ref{eq:usef}) to build an ensemble that applies the
majority voting-based rule has a maximum error rate of 1/2.
\end{prop}

\begin{proof}
The proof is provided in Appendix \ref{proof_prop:grbau}.
\end{proof}

The main problem with the above deterministic procedures is that they leave no opportunity to find better performing ensembles. Thus, we move on now to the more dynamic
stochastic strategies.
Keep in mind that since in our model the distribution of $q$ will be estimated, in any of the selection strategies we can exploit this knowledge as a stopping rule. Namely, even for the deterministic approaches, we can check whether the accuracy of the extracted ensemble is already attractive or whether we should continue and search for a better one. 

\section{Stochastic search algorithms}
\label{sec:selstrat2}

 As the deterministic selection strategies may have poor performance, we investigate stochastic algorithms to address our optimization problem.
Such randomized algorithms, where randomization only affects the order of the internal executions, 
produce the same result on a given input, which can cause the same problem we have found for the deterministic ones.  
In case of Monte Carlo (MC) algorithms \cite{TEMPO2007189}, the result of the simulations might change, but they produce the correct result with a certain probability. The accuracy of the MC approach depends on the number of simulations $N$; the larger $N$ is, the more accurate the algorithm will be. 
It is important to know how many simulations are required to achieve the desired accuracy. The error of the estimate of the probability failure is found to be 
$u_{1-\alpha/2} \sqrt{(1-\textrm{P}_f)/N\textrm{P}_f}$,
where $u_{1-\alpha/2}$ is the $1-\alpha/2$ quantile
of the standard normal distribution, and $\textrm{P}_f$ is the
true value of the probability of failure.

Simulated annealing (SA) \cite{du2016search}, as a variant of the Metropolis algorithm,
is composed of two main stochastic processes: generation and acceptance of solutions. 
SA is a general-purpose, serial search algorithm, whose 
solutions are close to the global extremum for an energy function within a polynomial upper bound for the computational time and are independent of the initial conditions.

To compare the MC method with SA for solving a knapsack problem,
we applied simulations for that scenario in which the deterministic approaches failed to find the most accurate ensemble, that is, when $D_1=(1-\beta,T)$, and $D_2=D_3=\ldots=D_n=(1/2+\varepsilon, T/n)$ with $
0<\beta<1/2, 0<\varepsilon<1/2$.
For this setup, we obtained that the precision of the MC method was only $11\%$, while SA found the most accurate ensemble in $96\%$ of the simulations.

Now, we introduce a novel search strategy that takes better advantage of our stochastic approach than, e.g., SA. This strategy builds ensembles using a randomized search technique and introduces a concept of usefulness for member selection, which better adapts to the ensemble energy than the classic one (\ref{eq:usef}).
Namely, in our proposed approach, the selection of the items for the ensemble is
based on the efficiency of the members determined in the following way: for
the $i$-th item with accuracy $p_{i}$ and execution time $t_{i}$, the system
accuracy $q(p_{i},t_{i})$ of the ensemble containing the maximal number 
of $i$-th items 
\begin{equation}
q(p_{i},t_{i})=\sum\limits_{k=0}^{\left\lfloor T/t_i\right\rfloor}
\binom{\left\lfloor T/t_i\right\rfloor }{k}{p_{i}}^{k}(1-p_{i})^{\left\lfloor T/t_i
\right\rfloor -k}  \label{efficiency}
\end{equation}%
characterizes the efficiency (usefulness) of the $i$-th item, instead of 
(\ref{eq:usef}).

A greedy algorithm for an optimization problem always chooses the item that
seems to be the most useful at that moment. In our selection method, a discrete
random variable depending on the efficiency values of the remaining items is
applied in each step to determine the probability of choosing an item from
the remaining set to add to the ensemble. Namely, in the $k$-th selection
step, if the items $i_1, \ldots, i_{k-1}$ are already in the ensemble, then
the efficiency values $q^{(k-1)}(p_i, t_i)$ of the remaining items are
updated to the maximum time of $T_k =T-\sum_{j=1}^{k-1} t_{i_j}$, where  $q^{(0)}(p_i, t_i)=q(p_i, t_i)$ and $T_0=T$.

The $i$-th item
is selected as the next member of the ensemble with the following
probability: 
\begin{equation}  \label{prob}
(\textrm{P}_{ens})_i^{(k)}=\frac{q^{(k-1)}(p_i, t_i)}{\sum\limits_{j}^{}
q^{(k-1)}(p_j, t_j)},
\end{equation}
where $i, j  \in {\cal N}
\backslash \{i_1, \ldots, i_{k-1}\}$. This
discrete random variable reflects that the more efficient the item is, the
more probable it is to be selected for the ensemble in the next step.

If $t_i > T_k$
for all $i \in {\cal N} 
\backslash \{i_1, \ldots, i_{k-1}\}$, then our stochastic process ends for the given search step
since there is not enough remaining time for any items. Then, we restart the process to extract another ensemble in the next search step. As a formal description of our proposed stochastic search method, see Algorithm \ref{alg:sajat}; notice that we evaluate the accuracy of ensembles with odd cardinalities only as in Lemma \ref{OddLemma}. A very important issue regarding both our approach and SA is the exact definition of the number of search steps, that is, a meaningful STOP parameter -- and also an escaping MAXSTEP one -- for Algorithm \ref{alg:sajat}. In our preliminary work \cite{7899637}, we have already tested the efficiency of our approach; however, we tested it empirically with an \textit{ad hoc} stopping rule. Now, in the forthcoming sections, we present how the proper derivation of the stopping parameters (STOP and MAXSTEP) can be derived.

\begin{algorithm}%
\caption{Proposed Stochastic searcH for EnsembLe Creation (SHErLoCk).}
\label{alg:sajat}
\begin{algorithmic}[1]%
\REQUIRE \begin{tabular}[t]{l} 
 Pool $\mathcal{D}=\ \left\{
\left(p_{i},t_{i}\right),~ i=1,\ldots,n \right\}$,\\ 
 Total allowed time $T$,\\
 Stochastic stopping value STOP,\\
 Maximum search steps MAXSTEP.
\end{tabular}

\smallskip
\hrule
\smallskip
\ENSURE An ensemble $\mbox{MAXENS}\subseteq\mathcal{D}$ to maximize system accuracy (\ref{accuracy}) within time $T$ as in (\ref{qTcondition}).
\hrule
\smallskip
\STATE
$\mbox{STEP}\leftarrow 0$, $\mbox{MAXENS}\leftarrow \emptyset$, $q_{\ell_0}\leftarrow 0$
\WHILE {STEP$<$MAXSTEP}
\STATE 
$H\leftarrow\mathcal{D}$, $\mbox{ENS}\leftarrow\emptyset$, $T'\leftarrow T$, $\mbox{SP}\leftarrow\emptyset$
\WHILE {$\exists (p_j,t_j)\in H:t_j \leq T-~~\sum\limits_{\mathclap{(p_k,t_k)\in \mbox{ENS}}} ~~~t_{k}~$}
\STATE $\forall(p_i,t_i) \in H$ calculate $q(p_i, t_i)$ by (\ref{efficiency})
\STATE $\forall(p_i,t_i) \in H$ calculate $ \textrm{P}_{ens_i}$ by (\ref{prob}) and 
\STATEx $SP\leftarrow\ SP\cup \{\textrm{P}_{ens_i}\} $
\STATE
Select a $(p_j,t_j)$ randomly from $H$ by \STATEx distribution $SP$
\IF {$t_j<T'$}
\STATE
 $ENS\leftarrow ENS \cup \{(p_j,t_j)\}$
 \STATE
  $H\leftarrow H\setminus\{(p_j,t_j)\}$
\STATE
  $T'\leftarrow T'-t_j$
\IF{$mod(size(\mbox{ENS}),2)=1$}
\STATE
\textrm{Calculate} $q_{\ell}(\textrm{ENS})$ \textrm{by (\ref{accuracy})}
\ENDIF
\IF {$q_{\ell_0} < q_{\ell}$}
\STATE
$q_{\ell_0} \leftarrow q_{\ell},~\textrm{MAXENS}\leftarrow \textrm{ENS}$
\ENDIF

\IF {$q_{\ell_0}   > STOP$}
\STATE \textbf{return} \textrm{MAXENS}
\ENDIF
\ENDIF
\ENDWHILE
\STATE
$\mbox{STEP} \leftarrow \mbox{STEP}+1$
\ENDWHILE
\STATE \textbf{return} {\let\textbf\relax\mbox{MAXENS}}

\end{algorithmic}
\end{algorithm}

\section{Stochastic estimation of ensemble energy}
\label{sec:stoc}

We need to examine and characterize the behavior of $q_{\ell}$ in (\ref
{accuracy}) to exploit these results to find and apply the proper stopping criteria in stochastic search methods.

Let $p\in\left[  0,1\right]$ be a random variable with mean $\mu_p$ and variance $\sigma^{2}_p$, where $p_{i}$ $(i=1,2,\ldots,n)$ are independent and identically distributed according to $p$, i.e., a sample.
Furthermore, let $\mu_{q_{\ell}}$ and $\sigma^2_{q_{\ell}}$ denote the mean and variance of the ensemble accuracy $q_{\ell}$, respectively.
In this case, it is seen that $\mu_p\leq1$ and a simple calculation shows that 
\begin{equation}
\mu_{q_{\ell}}=\sum_{k=\left\lfloor \frac{\ell}{2}\right\rfloor +1}^{\ell} \binom{\ell}{k}\mu_p
^{k}\left(  1-\mu_p\right)  ^{\ell-k}. \label{varhatoE}%
\end{equation}
The following lemma shows the basic properties of the mean and the variance of $q_{\ell}$.

\begin{lemma}
\label{Lemma_Eq_Varq}Let $p\in\left[  0,1\right]  $ be a random variable with
mean $\mu_p$ and variance $\sigma^{2}_p$. Consider the accuracy
(\ref{accuracy}), where $p_{i}$, $i=1,2,\ldots,n$ are i.i.d. random variables
distributed as $p$. Then,

\begin{enumerate}[wide, labelwidth=!, labelindent=0pt]
\item
\begin{equation}
\lim_{\ell\rightarrow\infty}\mu_{q_{\ell}}=\left\{
\begin{array}
[c]{ccc}%
0, & \text{if} & \mu_p\notin[1/2,1), \\
1/2, & \text{if} & \mu_p=1/2,\\
1, & \text{if} & \mu_p\in(1/2,1).
\end{array}
\right.   \label{limE}%
\end{equation}
Moreover, for odd $\ell$: if $\mu_p\in\left(  1/2,1\right)  $, then $\mu_{q_{\ell}}$
is increasing, and if $\mu_p\in\left(  0,1/2\right)  $, then $\mu_{q_{\ell}}$ is decreasing.

\item The variance of $q_{\ell}$ is expressed by
\begin{equation}
\begin{split}
& \sigma^2_{q_{\ell}}  =\sum_{k=k_{\ell}}^{\ell}\sum_{m=1}^{k}%
\sum_{\substack{h=  k_{\ell}-m}}^{\ell-k}
\delta\left(\ell,m,k\right)
\binom{\ell}{k}
\\
\times\binom{k}{m}%
&\binom{\ell-k}{h} 
s_{T}^{m} s_{TF}^{k-m+h}s_{F}^{\ell-k-h}-\left(   \mu_{q_{\ell}}\right)^{2},  \label{sz2}%
\end{split} 
 \end{equation} 
where $\delta\left(\ell,m,k\right)=\delta_{k_{\ell}-m\leq \ell-k}$, $s_{T}=\sigma^{2}_p+\mu^{2}_p$,  \\$s_{F}=\sigma^{2}_p+\left(  1-\mu_p\right)
^{2}$, $s_{TF}=\mu_p\left(  1-\mu_p\right)  -\sigma^{2}_p$, and $k_{\ell}=\left\lfloor \frac{\ell}{2}\right\rfloor +1$.

\item If $\mu_p\neq1/2$, $\mu_p\left(  1-\mu_p\right)  -\sigma^{2}_p>0,$ 
and
$s_{T}\neq1/2$, then
\begin{equation}
\lim_{\ell\rightarrow\infty}\sigma^2_{q_{\ell}} =0.
\label{Var_Limit}%
\end{equation}
If $s_{T}=1/2$, then the limit (\ref{Var_Limit}) is 1.
\end{enumerate}
\end{lemma}

\begin{proof}
See Appendix \ref{Proof_Lemma_Eq_Varq} for the proof.
\end{proof}

Notice that the condition $\ell\rightarrow\infty$ naturally assumes the same for the pool size with $n\rightarrow\infty$ in Lemma \ref{Lemma_Eq_Varq}. As a demonstrative example for the first part of Lemma \ref{Lemma_Eq_Varq}, see Figure \ref{ConvMu} regarding the three possible accuracy limits described in (\ref{limE}) with respective \emph{Beta}$(\alpha_p,\beta_p)$ distributions for $p$. 

\begin{figure}[!ht] 
\centering
\includegraphics[height=2.8in]{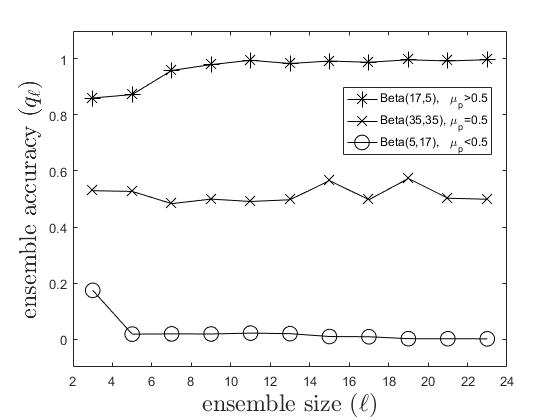}
\captionsetup{format=hang}
  \caption{Convergence of ensemble accuracies for member accuracies coming from different \emph{Beta}$(\alpha_p,\beta_p)$ distributions.}
\label{ConvMu}
\end{figure}

Now, to devise a stochastic model, we start with checking the possible distributions of the member accuracy values $p_i$ to estimate the ensemble accuracy. Then, we extend our model regarding this estimation by incorporating time information, as well. Notice that the estimation of the ensemble accuracy will be exploited to derive a stopping rule for the ensemble selection process. 

\subsection{Estimation of the distribution of member accuracies}

Among the various possibilities, we have found that the \emph{beta} distribution is a very good choice to analyze the distribution of member accuracies. The main reason is that \emph{beta} concentrates on the interval $[0,1]$, that is, it can exactly capture the domain for the smallest/largest accuracy. Moreover, the \emph{beta} distribution is able to provide  density functions of various shapes that often appear in practice. Thus, to start the formal description, let the variate $p$ be distributed as \emph{Beta}$\left(  \alpha_p,\beta_p\right)$ with density
\begin{equation}
b\left(  x;\alpha_p,\beta_p\right)  =\frac{x^{\alpha_p-1}\left(  1-x\right)
^{\beta_p-1}}{B\left(  \alpha_p,\beta_p\right)  },
\end{equation}
where $B\left(  \alpha_{p},\beta_{p}\right)  =\Gamma\left(  \alpha_{p}\right)
\Gamma\left(  \beta_{p}\right)  /\Gamma\left(  \alpha_{p}+\beta_{p}\right)  $. In this case,
\begin{equation}
\label{eqep}
\mu_p=\alpha_p/\left(  \alpha_p+\beta_p\right)  ,
\end{equation}
and $\mu_p\in\left(  1/2,1\right)  $ if and only if $\alpha_p>\beta_p$. If
$\alpha_p=\beta_p$, then $\mu_p=1/2$. In the case of $\alpha_p>\beta_p$, the mode is also
greater than $1/2$. 
The mode is infinite if $\beta_p<1$; therefore, we exclude
this situation and we assume from now on that
\begin{equation}
1<\beta_p<\alpha_p.
\end{equation}
The variance of $p$ is
\begin{equation}
\label{eqvarp}
\sigma^{2}_p=\frac{\alpha_p\beta_p}{\left(  \alpha_p
+\beta_p\right)  ^{2}\left(  \alpha_p+\beta_p+1\right)  }.
\end{equation}
Since $\mu_{q_{\ell}}$, and $\sigma_{q_{\ell}}^2$ depend on $\mu_p$, and
$\sigma^{2}_p $ according to
(\ref{varhatoE}) and (\ref{sz2}) respectively, one can calculate both of
them explicitly. The convergence of $\mu_{q_{\ell}}$ to 1 is fast if $\mu_p$ is close
to $1$, i.e., $\beta_p\ll\alpha_p$; for instance, if $\alpha_p=17$, $\beta_p=5$.
Simulations show that the speed of the convergence of $\sigma_{q_{\ell}}^2$ is exponential; hence, the usual square-root law does
not provide the Central Limit Theorem for $q_{\ell} $.

In practice, we perform a \emph{beta} fit on the $p_i$'s $(i=1,\dots,n)$. If a fit is found at least at the confidence level 0.95, we take the parameters $\alpha_p, \beta_p$ provided by the fit and calculate $\mu_p, \sigma^{2}_p$ by (\ref{eqep}) and (\ref{eqvarp}), respectively. If the \textsl{beta} fit is rejected, then $\mu_p$ and $\sigma^{2}_p$ are estimated from the $p_i$'s as the empirical mean and variance:
\begin{equation}
\label{empp}
\mu_p=\frac{1}{n}\sum\limits_{i=1}^n p_i,~
\sigma^{2}_p=\frac{1}{n-1}\sum\limits_{i=1}^n (p_i-\mu_p)^2.
\end{equation}
To simplify our further notation we do not indicate whether the mean and variance have been estimated from the fitted distribution or empirically.

\subsection{Adding time constraints to the model}
\label{sec:time}

Now, we turn to the case when together with the item accuracy $p_{i}$, we consider its running time $t_{i}$, as well. The common distribution of a random time is exponential, so let  $\tau$ be an exponential distribution with density
$\lambda\exp\left(  -\lambda t\right)  $.  If $p$ is distributed as \emph{Beta}$\left(  \alpha_p,\beta_p\right)$, then with setting $\lambda=1-p$ for a given $p$,
the distribution of $\lambda$ becomes \emph{Beta}$\left(  \beta_p,\alpha_p\right)$. 

This is a reasonable behavior of time because it is quite natural to assume that more accurate components require more resources such as a larger amount of computation times. On the other hand, the selection procedure becomes trivial, if, e.g., the time and accuracy are not inversely proportional, since then the most accurate member is also the fastest one; therefore, it should be selected first by following this strategy for the remaining members until reaching the time limit. For some other possible simple accuracy--time relations, see our preliminary work \cite{7899637}.

For a given time constraint $T$, consider the
random number $\ell_{T}$ such that
\begin{equation}
\sum_{j=0}^{\ell_{T}}\tau_{j}\leq T. 
\end{equation}
With the following lemma, our purpose is to provide an estimation $\widehat{{\ell}_{T}}$ for the expected size of the composed ensemble and incorporate this information in our stochastic characterization of $q_\ell$. 

\begin{lemma}
\label{LemmaTime}
Let $\tau$ be an exponential distribution with density $\lambda\exp\left(
-\lambda t\right)  $ under the condition that the parameter $\lambda$ is
distributed as \textsl{Beta}$\left(  \beta_p,\alpha_p\right)$, where $2<\beta_p<\alpha_p$.

\begin{enumerate}
\item Then, the expected time for the sum of $n$ variables is
\begin{equation}
\sum_{k=0}^{n}E\tau_{k}=n\left(  1+\frac{\alpha_p}{\beta_p-1}\right)
\end{equation}
with variance
\begin{equation}
Var\left(  \sum_{k=0}^{n}\tau_{k}\right)  =n\left(  1+\frac{\alpha_p}{\beta_p
-2}\right)  .
\end{equation}
This implies that the estimated number of ensemble members up to time $T$ is
$
\widehat{\ell_{T}}=\left\lceil T\frac{\beta_p-1}{\alpha_p+\beta_p-1}\right\rceil.
$

\item If the interarrival times $\tau_{j}$ correspond to a given $T$ and $\lambda$ generated from $\textsl{Beta}\left(  \beta_p,\alpha_p\right)$, then
\begin{equation}
\label{estN}
E(\ell_{T})=\frac{\beta_p}{\alpha_p+\beta_p}T.
\end{equation}

\item If each component of pair $\left(  \lambda_{j},\tau_{j}\right)  $ are independent
copies of $\lambda$ and $\tau_{j}$ corresponds to $\lambda_{j}$, then%
\begin{equation}
E(\ell_{T})=T\frac{\beta_p-1}{\alpha_p+\beta_p-1}.
\end{equation}
\end{enumerate}

In both cases 2) and 3), $\ell_{T}$ is distributed as Poisson with parameter
$T/E\tau_{1}$, which implies that $Var\left(  \ell_{T}\right)  =E(\ell_{T})$ and the
estimation of $\ell_{T}$ is
\begin{equation}
\widehat{{\ell}_{T}}=T/{\overline{\tau}}.
\end{equation}

\end{lemma}

\begin{proof}
See Appendix \ref{Proof_LemmaTime} for the proof. 
\end{proof}

So far, we have assumed that $p$ is distributed as \emph{beta} to calculate $\widehat{\ell_{T}}$ by Lemma \ref{LemmaTime}. 
If this is not the case, we consider the following simple and obvious calculation for the approximate number of $\ell$ under the time constraint $T$:
\begin{equation}
\label{nhat2}
\widehat{\ell_{T}}=\left\lceil nT\bigg/{\sum_{i=1}^{n}t_{i}}\right\rceil
=\left\lceil T/{\overline{t}}\right\rceil;
\end{equation}
another alternative to derive $\widehat{  \ell_{T}}$ in this case is discussed in section \ref{sec:disc}.
In either way it is derived, the value $\widehat{\ell_{T}}$ will be used in the stopping rule in our ensemble selection procedure; the proper details will be given next.

\subsection{Stopping rule for ensemble selection}
\label{subs:stop}

The procedure of finding  $\left(  {\cal L}_{0},\ell_{0}\right)$ is a selection task that is NP-hard. We propose an algorithm such that we stop the selection when the value of $q_{\ell}\left( {\cal L}\right)$ is sufficiently close to the possible maximum, which is not known. To be able to do so, we must give a proper stochastic characterization of $q_{\ell}$ by also settling on the calculation of $\mu_{q_\ell}$ and $\sigma^2_{q_\ell}$ via  Lemma \ref{Lemma_Eq_Varq}. First, notice that the values of $q_{\ell}\ $are in $\left(  0,1\right)$; indeed, it is positive and
\begin{equation}
\begin{split}
q_{\ell} &  =\sum_{k=\left\lfloor \frac{\ell}{2}\right\rfloor +1}^{\ell}\sum
_{\substack{I\subseteq {\cal N}\\\left\vert I\right\vert =k}}%
{\displaystyle\prod\limits_{i\in I}}
p_{i}%
{\displaystyle\prod\limits_{j\in {\cal N}\backslash I}}
\left(  1-p_{j}\right)  \\
&  <%
{\displaystyle\prod\limits_{j}}
\left(  p_{j}+\left(  1-p_{j}\right)  \right)  =1.
\end{split}
\end{equation}
For the case when $p_{i}$'s are \emph{beta} distributed, the product of independent \emph{beta} variates can be close to \emph{beta} again; see \cite{tang1984distribution}. We have also performed MC simulation and found that \emph{beta} distributions fit $q_\ell$ particularly well, compared to, e.g., the gamma, normal, Weibull, and extreme-valued distributions. Specifically, though the \emph{beta} behavior of $q_\ell$ was naturally more stable for \emph{beta} distributed $p_{i}$'s, the usual behavior of $q_\ell$ was also the same for non-\emph{beta} $p_{i}$'s. 

Thus, to provide a description of the stochastic behavior of $q$, we consider the following strategy. With a primary assumption on the $\emph{Beta}(\alpha_q, \beta_q)$ distribution of $q_\ell$, we calculate $\alpha_q$ and $\beta_q$ as
\begin{equation} \label{Beta_est}
\begin{split}
\alpha_q=\left(\frac{1-\mu_q}{\sigma_q^2}-\frac{1}{\mu_q}\right) \mu_q^2, ~~
\beta_q=\alpha_q\left(\frac{1}{\mu_q}-1\right).
\end{split}
\end{equation}
If time information is provided for the pool items, we calculate $\widehat{\ell_{T}}$ by Lemma \ref{LemmaTime}, and as a simpler notation, we will write $\widehat{\ell}$ from now on. If time information is not available, we will set $\widehat{\ell}=n$.

Next, we decide whether $q_\ell$ should be considered as \emph{beta} with requiring $1<\beta_q<\alpha_q$ to be fulfilled to have a mode that is larger than 1/2 and finite. If this condition does not hold, we reject the \emph{beta} behavior of $q_\ell$, and based on simulations, we characterize it as a normal distribution and stop the search if 
\begin{equation}
\label{stopnorm}
q_{\ell}\geq
\kappa_{0.9}\sigma_{q_{\widehat{\ell}}}/\sqrt{\widehat{\ell}}+ \mu_{q_{ \widehat{\ell}}}= STOP,
\end{equation}
where $\kappa_{0.9} $ is the 0.9 quantile of the standard normal distribution.
Otherwise, when $q_\ell$ is considered \emph{beta}, we calculate the mode $\nu$ of \emph{Beta}($\alpha_q$, $\beta_q$) for $q_{\ell}$ as
\begin{equation}
\nu=\frac{\alpha_q-1}{\alpha_q+\beta_q-2},
\label{loc_mu}
\end{equation}
and the Pearson's first skewness coefficient as \begin{equation}
\gamma=\frac{1-\nu}{ {\sigma_{q_ {\widehat{\ell}}}}} \label{loc_gamma}.
\end{equation} 
Then, we  use Table \ref{tabl} to select the appropriate probability value $\varrho_{q_{\widehat{\ell}}}$; the entries are determined by simulation in the case of
$ 2 \leq\beta_q < \alpha_q$.

 We stop the selection when the ensemble accuracy reaches the value of the inverse cumulative distribution $F_{\alpha_q ,\beta_q}^{-1}(  \varrho_{q_{\widehat{\ell}}})$ of $\emph{Beta}(\alpha_q,\beta_q)$ in the given probability, that is, when
\begin{equation}
\label{stopBeta}
q_{\ell}\geq F_{\alpha_q ,\beta_q}^{-1}( \varrho_{q_{\widehat{\ell}}})=STOP.
\end{equation}

\begin{table}[ht] 
\caption{Probability values $\varrho_{q_{\widehat{\ell}}}$ for stopping thresholds for different skewness coefficients $\gamma$.}
\label{tabl}
\begin{center}
\begin{tabular}{|c|c|}
\hline
 $\gamma$ & $\varrho_{q_{\widehat{\ell}}}$ \\
\hline
$\gamma\leq 1$ & $0.6$  \\
\hline
$1<\gamma\leq 2.5$ & $0.8$
\\
\hline
$2.5<\gamma\leq 3.5$& $0.9$ 
\\
\hline
$3.5<\gamma$& $0.95$  \\
\hline
\end{tabular}
\end{center}
\end{table}

In either via (\ref{stopnorm}) or (\ref{stopBeta}), an estimation for the ensemble accuracy is gained; we obtain a STOP value to stop the stochastic search. However, there is some chance that STOP is not exceeded, though in our experiments it has never occurred. Thus, to avoid an infinite loop, we consider a maximum allowed step number MAXSTEP as an escaping stopping rule. Namely, to obtain MAXSTEP, we apply Stirling's approximation
\begin{equation}
\mbox{MAXSTEP}=\binom{n}{\widehat{\ell}}
\sim n^{\widehat{\ell}}/{\widehat{\ell}!},\label{final_stop}
\end{equation}
assuming that $\widehat{  \ell}/n\rightarrow0$. %
This is a reasonable app\-roach since  $\widehat{\ell}$ is calculated according to Lemma \ref{LemmaTime} or (\ref{nhat2}).
The formal description of our proposed ensemble selection method is enclosed in Algorithm \ref{alg:propsearch}. 

\begin{algorithm}
\caption{Proposed Ensemble Creation Method.}
\label{alg:propsearch}
\begin{algorithmic}[1]%
\REQUIRE [NO-TIME]: Pool $\mathcal{D}=\left\{p_i \right\}_{i=1}^n$. 
\smallskip
\hrule
\smallskip
\REQUIRE [TIME]:
\begin{tabular}[t]{l} 
 Pool $\mathcal{D}=\left\{
\left(p_{i},t_{i}\right)\right\}_{i=1}^n$,\\ 
 Total allowed time $T$.\\ 
\end{tabular}
\smallskip
\hrule
\smallskip
\ENSURE An ensemble $\mbox{MAXENS}\subseteq\mathcal{D}$ to maximize system accuracy (\ref{accuracy}) within time $T$ as in (\ref{qTcondition}).
\hrule
\smallskip

\STATE
Calculate the mean $\mu_p$ and std $\sigma_{p}$ 
for $\{p_i\}_{i=1}^n$ 
\STATEx by (\ref{eqep}) and (\ref{eqvarp}) (if a \emph{beta} fits to $p$) or 
\STATEx empirically (if $p$ is not \emph{beta}) by (\ref{empp})

 \Switch{\textbf{Input}}
    \Case{NO-TIME}
      \STATE $\widehat{\ell}\gets n$
    \EndCase
    \Case{TIME}
      \STATE Estimate \# of members $\widehat{\ell}$ for $T$ by 
      \STATEx Lemma \ref{LemmaTime} if a 
      \emph{beta} fits to $p$, or by (\ref{nhat2}) 
      \STATEx if $p$ is not \emph{beta}
       \EndCase
  \EndSwitch

\STATE  Calculate
$\mu_{q_{\widehat{\ell}}}$ by (\ref{varhatoE}) and $\sigma^2_{q_{\widehat{\ell}}}$ by (\ref{sz2})

\STATE Calculate $\alpha_q$, $\beta_q$ by (\ref{Beta_est})
\IF {$1<\beta_q<\alpha_q$} 
\STATE{Calculate cdf. $F_{\alpha_q,\beta_q}$, 
$\nu$, $\gamma$, $\varrho_{q_{\widehat{\ell}}}$ by (\ref{loc_mu}), (\ref{loc_gamma}) 
\STATEx and Table \ref{tabl}, and adjust STOP with (\ref{stopBeta})}
\ELSE
\STATE{Adjust STOP with (\ref{stopnorm})}
\ENDIF
\STATE
Calculate MAXSTEP by (\ref{final_stop})
 \Switch{Input}
    \Case{NO-TIME}
      \STATE Compose ensemble by SA using STOP
      \STATEx and MAXSTEP for the stopping rule
    \EndCase
    \Case{TIME}
      \STATE Compose ensemble either by Algorithm \ref{alg:sajat}
      \STATEx (SHErLoCk) or SA using STOP and  
      \STATEx MAXSTEP for the stopping rule
    \EndCase
  \EndSwitch

\end{algorithmic}%
\end{algorithm}%

Before providing our detailed empirical results in section \ref{sec:empanal}, in Table \ref{tabl2} we summarize our findings for Algorithm \ref{alg:propsearch} on simulations. Namely, in two respective tests with $i=1,\dots,30$ and $i=1,\dots,100$, we have generated the $p_i$'s to come from $\emph{Beta}(17,5)$ and the execution times $t_i$ from conditional exponential distributions with parameters $\lambda=1-p_i$. The time constraint $T$ was set in seconds to $30\%$ of the total time $\sum_{i=1}^{30} t_i$ for the first, and $20\%$ of $\sum_{i=1}^{100} t_i$ for the second test. Both tests were repeated 100 times, and we have taken the averages of the obtained precisions. 
As our primary aim, we have checked whether the stopping rule of the stochastic search indeed led to a reasonable computational gain. For the sake of completeness, in Table \ref{tabl2} we have also shown the results regarding letting the search continue in the long run (stopped by MAXSTEP), though in each of our tests, the STOP value has been exceeded much earlier. Secondarily, 
we have compared SA with our selection method SHErLoCk given in Algorithm \ref{alg:sajat}. For Table \ref{tabl2}, we can conclude that applying our stopping rule by using STOP saved considerable computational time compared with the exhaustive search that culminated by stopping it with MAXSTEP with a negligible drop in accuracy. Moreover, our approach has found efficient ensembles quicker than SA. These impressions have also been confirmed by the empirical evaluations on real data described in the next section.

\begin{table}[ht]

\caption{Result of Algorithm \ref{alg:propsearch} on simulations.}
\label{tabl2}
\centering
\tabcolsep=0.11cm
\begin{tabular}{|c||c|c||c|c|}
\hline

Search &  \multicolumn{2}{c||}{Ensemble accuracy} & \multicolumn{2}{c|}{Comp. time (secs)} \\ \cline{2-5}

method & MAXSTEP & STOP & MAXSTEP & STOP 
 \\
\hline
SHErLoCk
($n$=30)&
$99.56\%$ & $ 99.39\%$  & $ 60.03 $& $ 0.08 $\\

\hline
SA ($n$=30)& $  98.97\%$ &  $98.91 \%$ & $87.40$ & $ 0.30
$ 
\\
\hline
SHErLoCk 
($n$=100) & $99.66\%$ &  $  99.61\%$ & 
$294.58$ & $ 1.54$ 

\\
\hline
SA ($n$=100)& $99.38\%$ &  $99.37\%$ & $ 638.39$ & $1.58 $ \\
\hline
\end{tabular}
\end{table}

\section{Empirical analysis}
\label{sec:empanal}

In this section, we demonstrate the efficiency of our models through an exhaustive experimental test on publicly available data. Our first experiment considers the possibility of organizing competing approaches with different accuracies into an ensemble. In this scenario, accuracy values correspond to final scores of participants of Kaggle\footnote{www.kaggle.com} challenges without cost/time information provided. Our second setup for ensemble creation considers machine learning-based binary classifiers as possible members; the performance evaluation is performed on several UCI Machine Learning Repository \cite{Dua:2017} datasets with the training times considered as costs.    

\subsection{Kaggle challenges}

Kaggle is an open online platform for predictive modeling and analytics competitions with the aim of solving real-world machine learning problems provided by companies or users. The main idea behind this crowd-sourcing approach is that a countless number of different strategies might exist to solve a specific task, and it is not possible to know beforehand which one is the most effective. Though primarily only the scores of the participating algorithms can be gathered from the Kaggle site, as a possible future direction, we are curious regarding whether creating ensembles from the various strategies could lead to an improvement regarding the desired task. 

Not all the Kaggle competitions are suitable to test our models since in the current content, we focus on majority voting-based ensemble creation. Consequently, we have collected only such competitions and corresponding scores where majority voting-based aggregation could take place. More precisely, we have restricted our focus only to such competition metrics based on which majority voting can be realized. Such metrics  
include quadratic weighted kappa,  area under the ROC curve (AUC), log loss, normalized Gini coefficient. For concrete competitions where these metrics were applied, we analyze the following ones: Diabetic Retinopathy Detection\footnote{www.kaggle.com/c/diabetic-retinopathy-detection}, DonorsChoose.org Application Screening \footnote{www.kaggle.com/c/donorschoose-application-screening}, Statoil/C-CORE Iceberg Classifier Challenge\footnote{www.kaggle.com/c/statoil-iceberg-classifier-challenge}, WSDM - KKBox's Churn Prediction Challenge\footnote{www.kaggle.com/c/kkbox-churn-prediction-challenge}, and Porto Seguro’s Safe Driver Prediction\footnote{www.kaggle.com/c/porto-seguro-safe-driver-prediction/data}.

For our analytics, on the one hand it is interesting to observe the distribution of the final score of the competitors, which is often affected by the volume of the prize money offered to the winner. Moreover, accuracy measurement is usually scaled to the interval $[0,1]$, with 0 for the worst and 1 for the perfect performance, which allows us to test our results regarding the \emph{beta} distributions. As a drawback of Kaggle data, access to the resource constraints corresponding to the competing algorithms (e.g., training/execution times) is rather limited; such data are provided for only a few competitions, primarily in terms of execution time interval.

Thus, to summarize our experimental setup, we interpret the competing solutions of a Kaggle challenge as the $\{D_1, D_2, \dots, D_n\}$ pool, where the score of $D_i$ is used for the accuracy term $p_i\in [0,1]$ in our model. Then, we apply a \emph{beta} fit for each investigated challenge to determine whether a \emph{beta} distribution fits the corresponding scores or not. If the test is rejected, we can still use the estimation for the joint behavior $q$ using (\ref{empp}) and (\ref{nhat2}). If the \emph{beta} test is accepted, we can also apply our corresponding results using (\ref{eqep}), (\ref{eqvarp}), and Lemma \ref{LemmaTime}. Notice that reliably fitting a model for the scores of the competitors might lead to a better insight of the true behavior of the data of the given field, also for the established expectations there. 

As observed from Table \ref{tab:kaggle_m30}, SA was able to stop much earlier with a slight loss in accuracy using the suggested stopping rule (STOP) in finding the optimal ensemble. Our approach SHErLoCk given in Algorithm \ref{alg:sajat} has been excluded from this analysis since no cost information was available.

\begin{table*}[!ht]

\caption{Ensemble accuracies on the Kaggle datasets found by simulated annealing (SA).}
\label{tab:kaggle_m30}

\centering
\begin{tabular}{|c||c||c||c||c|}
\hline
Dataset &  \multicolumn{2}{c||}{Ensemble accuracy} & \multicolumn{2}{c|}{Computational time (secs)} \\ \hline
 Name &
 {MAXSTEP}& \multicolumn{1}{c||}{STOP}& {MAXSTEP}& \multicolumn{1}{c|}{STOP} \\
\cline{2-5}
\hline
Diabetic Retinopathy Detection  &  94.34\% &   93.19\% &   194.12& 1.31  \\ 
\hline
DonorsChoose.org Application Screening  & 94.78\% & 91.96\%  & 206.89 & 1.67 \\
\hline
Statoil/C-CORE Iceberg Classifier Challenge &  88.42\% &87.76\% &  191.91&  2.23\\ 
\hline
WSDM - KKBox's Churn Prediction Challenge &  96.96\% & 96.32\% & 203.88 & 1.45  \\ 
\hline
Porto Seguro’s Safe Driver Prediction & 92.99\% &89.98\% &  214.28 & 1.95 \\ 
\hline
\multicolumn{1}{|c||}{Average}  & 92.29\% & 90.45\%  & 202.21 & 1.72 \\
\hline

\end{tabular}
\end{table*}

\subsection{Binary classification problems}

The UCI Machine Learning Repository \cite{Dua:2017} is a popular platform to test the performances of machine learning-based approaches, primarily for classification purposes. A large number of datasets are made publicly available here among which our models can be tested on binary classification ones. That is, in this experiment, the members $D_1, D_2, \dots, D_n$ of a pool for ensemble creation are interpreted as binary classifiers, whose outputs can be aggregated by the majority voting rule. Using the ground truth supplied with the datasets, the accuracy term $p_i\in [0,1]$ stands for the classification accuracy of $D_i$. 

The number of commonly applied classifiers is relatively low; therefore to increase the cardinality of the pool, we have also considered a synthetic approach in a similar way to \cite{CAVALCANTI201638}. Namely, we have trained the same base classifier on different training datasets, by which we can synthesize several "different" classifiers. Naturally, this method is able to provide more independent classifiers only if the base classifier is unstable, i.e., minor changes in the training set can lead to major changes in the classifier output; such an unstable classifier is, for example, the perceptron one. 

To summarize our experimental setup for UCI binary classification problems, we have considered base classifiers perceptron \cite{Perceptron}, decision tree \cite{DTree}, Levenberg-Marquardt feedforward neural network \cite{Levenberg}, random neural network\cite{RandomNN}, and discriminative restricted Boltzmann machine classifier \cite{BoltzmanNN} for the datasets MAGIC Gamma Telescope \cite{BOCK2004511}, HIGGS \cite{Baldi2014SearchingFE}, EEG Eye State \cite{Dua:2017}, Musk (Version 2) \cite{NIPS1993_781}, and Spambase \cite{Dua:2017}; datasets of large cardinalities were selected to be able to train synthetic variants of base classifiers on different subsets. To check our models for different numbers of possible ensemble members, the respective pool sizes were set to $n=30$ and $n=100$; the necessary number of classifiers has been reached via synthesizing the base classifiers with training them on different subsets of the training part of the given datasets. In contrast to the Kaggle challenges, in these experiments we were able to retrieve meaningful cost information to devise a knapsack scenario. Namely, for a classifier $D_i$, its training time was adjusted as its cost $t_i$ in our model. Notice that for even the same classifier, it was possible to obtain different $t_i$ values with training its synthetic variants on datasets of different sizes. 
Using this time information, for the estimated size $\widehat{\ell}$ of the optimal ensemble, we could use Lemma \ref{LemmaTime} for $n=30$, while (\ref{nhat2}) for the case $n=100$.

As clearly visible from Tables \ref{tab:m30} and \ref{tab:uci_m100}, our stochastic search strategy SHErLoCk described in Algorithm \ref{alg:sajat} was reasonably faster than SA and slightly dominant in accuracy, as well. Moreover, it can be observed again that applying the stopping rule with the threshold STOP led to an enormous computational advantage for either search strategies with only a small drop in accuracy.

\begin{table*}[!ht]

\caption{Comparing simulated annealing (SA) with the proposed search strategy (SHErLoCk) on binary classification problems of UCI datasets using an ensemble pool of $n=30$ classifiers.}
\label{tab:m30}

\centering
\begin{tabular}{|c|c||c|c||c|c||c|c||c|c|}
\hline
\multicolumn{2}{|c||}{Dataset} &  \multicolumn{4}{c||}{Ensemble accuracy} & \multicolumn{4}{c|}{Computational time (secs)} \\ \hline
\multirow{2}{*}{Name} & \multirow{2}{*}{Size} & \multicolumn{2}{c||}{MAXSTEP}& \multicolumn{2}{c||}{STOP}& \multicolumn{2}{c||}{MAXSTEP}& \multicolumn{2}{c|}{STOP} \\
\cline{3-10}
 &  & SA & SHErLoCk & SA & SHErLoCk  & SA & SHErLoCk & SA & SHErLoCk  \\ 
\hline
MAGIC &  19 020 &  99.34\% & 99.57\% &   99.17\% & 99.29\%  & 171.9
 & 56.47 & 0.41 & 0.18  \\ 
\hline
Spambase &  4 601 &99.68\% &  99.76\% & 98.95\% &  98.77\% &100.98  & 67.90 & 1.67 & 0.39 \\
\hline
HIGGS &  20 000 & 78.03\% & 77.99\% &77.59\% &77.63\%& 158.90 & 69.46&  2.73& 0.93 \\ 
\hline
EEG &  14 980 & 98.71\% & 98.80\% & 95.62\% & 97.43\% & 345.57 &  90.20
& 0.47 & 0.38 \\ 
\hline
Musk &  6 598 & 99.95\% & 99.99\% &99.96\% &  99.98\% &  178.21&  58.89
  & 0.33  & 0.32 \\ 
\hline
\multicolumn{2}{|c||}{Average} & 95.15\% & 95.07\% & 94.26\% & 94.62\% & 191.11& 68.06  & 1.12&0.44\\
\hline

\end{tabular}
\end{table*}

\begin{table*}[!ht]

\caption{Comparing simulated annealing (SA) with the proposed search strategy on binary classification problems of UCI datasets using an ensemble pool of $n=100$ classifiers.}
\label{tab:uci_m100}

\centering
\begin{tabular}{|c|c||c|c||c|c||c|c||c|c|}
\hline
\multicolumn{2}{|c||}{Dataset} &  \multicolumn{4}{c||}{Ensemble accuracy} & \multicolumn{4}{c|}{Computational time (secs)} \\ \hline
\multirow{2}{*}{Name} & \multirow{2}{*}{Size} & \multicolumn{2}{c||}{MAXSTEP}& \multicolumn{2}{c||}{STOP}& \multicolumn{2}{c||}{MAXSTEP}& \multicolumn{2}{c|}{STOP} \\
\cline{3-10}
 &  & SA & SHErLoCk& SA & SHErLoCk  & SA & SHErLoCk & SA & SHErLoCk \\ 
\hline
MAGIC &  19 020 &  99.57\% & 99.59\% &   99.19\% & 99.34\%  & 349.62 & 194.12& 1.31 & 1.32 \\ 
\hline
Spambase &  4 601 &99.79\% & 99.78\% & 98.96\% &  98.88\% &390.89  & 206.89 & 1.67 & 1.41 \\
\hline
HIGGS &  20 000 & 78.08\% & 78.16\% &77.79\% &77.73\%& 378.56 & 191.91&  2.23& 2.11
  \\ 
\hline
EEG &  14 980 & 98.96\% & 98.96\% & 97.32\% & 97.64\% & 453.59 &  203.88
& 1.45 & 1.33 \\ 
\hline
Musk &  6 598 & 99.98\% & 99.99\% &99.98\% &  99.98\% &  475.71&  214.28 & 1.95  & 1.56 \\ 
\hline
\multicolumn{2}{|c||}{Average} & 95.28\% & 95.29\% & 94.65\% & 94.71\% & 409.67 & 202.21 & 1.72 & 1.55\\
\hline

\end{tabular}
\end{table*}

\subsection{Optic disc detection}

The majority voting rule can be applied in a problem to aggregate the outputs of single object detectors in the spatial domain \cite
{Hajdu2013jspatialvoting}; the votes of the members are given in terms of single pixels as candidates for the centroid of the desired object. In this extension, the shape of the desired object defines a geometric constraint, which should be met by the votes that can be aggregated. In 
\cite{Hajdu2013jspatialvoting}, our practical example relates to the detection of a disc-like anatomical component, namely the optic disc (OD) in retinal images. Here, the votes are required to fall inside a disc of diameter $d_{OD}$ to vote together. As more false regions are possible to be formed, the correct decision can be made even if the true votes are not in the majority, as in Figure \ref{OD_detection}. The geometric restriction transforms 
(\ref{accuracy}) to the following form: 
\begin{equation}
q_{\ell}({\cal L})=\sum\limits_{k=0}^{\ell}p_{\ell,k}\left( {\underset{|{\cal I}|= k}{\sum\limits_{\cal I\subseteq {\cal L}}}}
\prod\limits_{i\in \cal I}p_{i}\prod\limits_{j\in {\cal L}\setminus
\cal I}\left( 1-p_{j}\right) \right) .  \label{eq:spatial}
\end{equation}%
In (\ref{eq:spatial}), the terms $p_{\ell,k}$ describe the probability that a correct decision is made by supposing that we have $k$ correct votes out of $\ell$. For the terms $p_{\ell,k}$ $(k=0,1,\dots,\ell)$, in general, we have that $0\leq p_{\ell,0}\leq p_{\ell,1}\leq\dots\leq p_{\ell,\ell}\leq 1$.
\begin{figure}[!htb]
\begin{center}
{\includegraphics[height=4.5cm]{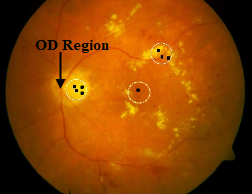}}
\end{center}
\caption[OD detection.]{Successful OD detection with the same number of correct/false ensemble member responses.}
\label{OD_detection}
\end{figure}

In our experiments, the pool consists of eight OD detector algorithms with the following accuracy and running time values: 
$\{(p_i,t_i)\}_{i=1}^8=\{(0.220,31), (0.304,38),(0.319,34), (0.643,69), \\ (0.754,11),(0.765,7),(0.958,21), (0.976,90) \}$ 
with $\sum_{i=1}^8 t_i=301$ secs. 
We can apply our theoretical foundation with some slight modifications to solve the same kind of knapsack problem for the variant (\ref{eq:spatial}), 
transforming the model to reflect the multiplication with the terms $p_{\ell,k}$.

We have empirically derived the values  
$p_{8,k}=\{ 0,0.11,0.70,0.93,0.99,1.00,1.00,1.00,1.00\}$ 
for (\ref{eq:spatial}) in our task.
To adopt our approach by following the logic of Algorithm \ref{alg:propsearch}, we need to determine a STOP value for the search based on $\mu_p$ and $\sigma_p$ (calculated by (\ref{empp})), and 
$\widehat{\ell}$ (calculated by (\ref{nhat2})). However, since now the energy function is transformed by the terms $p_{\ell,k}$ in (\ref{eq:spatial}), we must borrow the corresponding theoretical results from \cite{tiba2019optimizing} to derive the mean $\mu_{q_{\widehat{\ell}}}$ 
instead of (\ref{varhatoE}) proposed in Algorithm \ref{alg:propsearch}. Accordingly, we had to find a continuous function $\cal F$ that fit to the values $p_{\ell,k}$, which was evaluated by regression and resulted in ${\cal F} (x)={b}/({b+{x}^a/({1-x})^a})$ with $a=-3.43$ and $b=101.7$, as also plotted in Figure \ref{pnkcurve}.
Now, by using Theorem 1 from \cite{tiba2019optimizing}, we have gained $\mu_{q_{\widehat{\ell}}}={\cal F}(\mu_p)$.

\begin{figure}[!ht] 
\centering
\includegraphics[height=5cm,width=8cm]{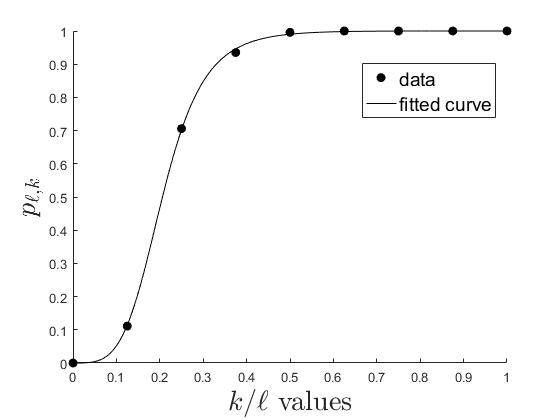}
\caption[OD detection.]{Determining the constrained majority voting probabilities $p_{\ell,k}$ for our OD detector ensemble.}
\label{pnkcurve}
\end{figure} 

For our experiment to search for the best ensemble, we have set the time constraint to be 80\% of the total running time, with $T= 4(\sum_{i=1}^8 t_i)/5$.
For this setup, we could estimate 
$\widehat{\ell}=7$ and $\mu_{q_{\widehat{\ell}}}=0.969$ for the expected ensemble size and mean accuracy, respectively.
Then, these values have been considered for Algorithm \ref{alg:propsearch} to compare the performance of our stochastic search method SHErLoCk with SA.
As shown in Table \ref{tabl6}, our search strategy outperformed SA also for the object detection problem both in accuracy and computational time.

\begin{table}[ht]

\caption{Comparing SA with the proposed search strategy SHErLoCk on the OD detection problem.}
\label{tabl6}
\centering
\begin{tabular}{|c||c|c|}
\hline

Search &  \multicolumn{1}{c|}{Ensemble accuracy} & \multicolumn{1}{c|}{Comp. time (secs)} \\ \cline{2-3}

method  & STOP & STOP 
 \\
\hline
SHErLoCk
& $ 99.45\%$  &  $ 0.07 $\\

\hline
SA 
&  $99.43 \%$  & $ 0.16$ 
\\
\hline

\end{tabular}
\end{table}

\section{Discussion}
\label{sec:disc}

For the approximate number $\widehat{{\ell}_{T}}$ of the ensemble size, we have considered (\ref{nhat2}) when the member accuracy $p$ is not a \emph{beta} distribution. As an alternative, notice that 
it is known that for a given $\lambda$, $T$ and an independent exponential distributed
$\tau_{j}$, $\ell_{T}$ is distributed as Poisson with parameter
$\lambda T$. We can use Lemma \ref{LemmaTime} and conclude that for a starting size of the ensemble, one may choose $\widehat{{\ell}_{T}}$ such that the remaining possible values are beyond the 5\% error. It follows that we apply formula either 
\begin{equation}
\textrm{P}\left(  \ell_{T}>m_{0.05}\right)  = 0.05,
\label{nTe_Poisson} %
\end{equation}
where $m_{0.05}$ is the upper quantile of the Poisson distribution with
parameter  $T/\sum_{i=1}^n p_i$, or use the normal
approximation to the Poisson distribution
\begin{equation}
\frac{\ell_{T}-T/\sum_{i=1}^n p_i -0.5}{\sqrt{T/\sum_{i=1}^n t_i}}>1.64,
\end{equation}
which provides us the inequality
\begin{equation}
\ell_{T}>\frac{T}{\sum_{i=1}^n t_i}+0.5+1.64\sqrt{\frac{T}{\sum_{i=1}^n t_i}}=\widehat{\ell_{T}}.
\label{loc_n}
\end{equation}
In our experiments we have used (\ref{nhat2}) instead of (\ref{loc_n}) to obtain $\widehat{\ell_{T}}$, since the latter provided slightly too large estimated size values. However, for other scenarios, it might be worthwhile to try (\ref{loc_n}), as well.

As some additional arguments, we call attention to the following issues regarding those elements of our approach that might need special care or can be adjusted differently in other scenarios:
\begin{itemize}
    \item We have assumed independent member accuracy behavior, providing solid estimation power in our tests. However, in the case of strong member dependencies, deeper discovery of the joint behavior might be needed.
    \item Stirling's approximation considered in (\ref{final_stop}) may provide values that are too small for the parameter MAXSTEP in the case of small pools. Since this is an escape parameter, a sufficiently large value should be selected in such cases instead.
    \item The time profile $\lambda=1-p$ considered in section \ref{sec:time} is suited to our data; however, any other relationship between the member accuracy and time can be considered. Nevertheless, the proper derivation of the estimation of the ensemble accuracy might be slightly more laborious.
    \item In (\ref{stopnorm}), we have used a one-tailed (left-side) hypothesis since $q_\ell$ was close to 1. However, if it is not that close to 1, a two-tailed hypothesis can be meaningful, as well. Furthermore, if $\mu_q$ is even smaller (say 0.7), then we can search above this mean by considering a right-side hypothesis. 
  
\end{itemize}

\ifCLASSOPTIONcompsoc
  \section*{Acknowledgments}
\else
  \section*{Acknowledgment}
\fi

This work was supported by the project EFOP-3.6.2-16-2017-00015 supported by the European Union, co-financed by the European Social Fund.

\ifCLASSOPTIONcaptionsoff
  \newpage
\fi

\bibliographystyle{IEEEtran}
\bibliography{main}

\hfill\newpage

\pagenumbering{gobble}

\appendices

\section*{Appendices}

\section{Proof for Lemma \ref{OddLemma}.}
\begin{proof}
\label{Proof_Odd_Lemma}
 Consider a subset $\mathcal{K}$ when  $\mathcal{K}%
=\left\{  1,2,\ldots,2\ell\right\}  $ (otherwise we can renumerate $p_{i}$). We have
\begin{multline}
q_{2\ell}\left(  \mathcal{K}\right)  =\sum_{k=\ell+1}^{2\ell}\sum
_{\substack{I\subseteq\mathcal{K}\\\left\vert I\right\vert =k}}%
{\displaystyle\prod\limits_{i\in I}}
p_{i}%
{\displaystyle\prod\limits_{j\in\mathcal{K}\backslash I}}
\left(  1-p_{j}\right)  \\
=\sum_{k=\ell+1}^{2\ell}\sum_{\substack{I\subseteq\mathcal{K}\\\left\vert
I\right\vert =k}}Q_{2\ell,k}\left(  \mathcal{K},I\right)  ,
\end{multline}
where
\begin{equation}
Q_{2\ell,k}\left(  \mathcal{K},I\right)  =%
{\displaystyle\prod\limits_{i\in I}}
p_{i}%
{\displaystyle\prod\limits_{j\in\mathcal{K}\backslash I}}
\left(  1-p_{j}\right),
\end{equation}
that is, we consider a subset $\mathcal{K}\subseteq\mathcal{N}$ with
$\left\vert \mathcal{K}\right\vert =2\ell$, and $Q_{2\ell,k}\left(
\mathcal{K},I\right)  $ is calculated for an index set $I\subseteq\mathcal{K}$
with $\left\vert I\right\vert =k$. Now, choose an index $a$ from the set
$\mathcal{N}\backslash\mathcal{K}$, i.e., $a>2\ell$, and obtain
\begin{equation}
Q_{2\ell,k}\left(  \mathcal{K},I\right)  =Q_{2\ell,k}\left(  \mathcal{K}%
,I\right)  p_{a}+Q_{2\ell,k}\left(  \mathcal{K},I\right)  \left(
1-p_{a}\right)  .
\end{equation}
The term $Q_{2\ell,k}\left(  \mathcal{K},I\right)  p_{a}=Q_{2\ell
+1,k+1}\left(  \left\{  \mathcal{K},a\right\}  ,\left\{  I,a\right\}  \right)
$ and $\ Q_{2\ell,k}\left(  \mathcal{K},I\right)  \left(  1-p_{a}\right)
=Q_{2\ell+1,k}\left(  \left\{  \mathcal{K},a\right\}  ,I\right)  $; therefore,
\begin{multline}
q_{2\ell}\left( {\cal{K}}\right)   
=\sum_{k=\ell+1}^{2\ell}\sum_{\substack{I\subseteq
{\cal{K}}\\\left\vert I\right\vert =k}}Q_{2\ell,k}\left(  {\cal{K}},I\right)  \\
=\sum_{k=\ell+1}^{2\ell}\left(  \sum_{\substack{I\subseteq {\cal{K}}\\\left\vert
I\right\vert =k}}Q_{2\ell+1,k+1}\left(  \left\{  {\cal{K}},a\right\}  ,\left\{
I,a\right\}  \right)  + \right. \\
\left. +\sum_{\substack{I\subseteq {\cal{K}}\\\left\vert I\right\vert
=k}}Q_{2\ell+1,k}\left(  \left\{  {\cal{K}},a\right\}  ,I\right)  \right)  \\
<\sum_{k=\ell+1}^{2\ell}\sum_{\substack{I\subseteq {\cal{K}}\\\left\vert I\right\vert
=k}}Q_{2\ell+1,k}\left(  {\cal{K}},I\right)  =q_{2\ell+1}\left(  {\cal{K}},a\right)  ,
\end{multline}
since $q_{2\ell+1}\left(  \mathcal{K}\right)  $ includes some extra additional
terms, say $Q_{2\ell+1,k}\left(  \left\{  \mathcal{K},a\right\}  ,I\right)  $,
where $I$ contains $a$. Regarding that $n$ is odd in the series of $q_{\ell}\left(
\mathcal{K}\right)$, there will be an element with odd $\ell$ following an
element of even $\ell$ and the lemma is proved for odd $n$. For the case
when $n$ is even, we consider the $q_{n}\left(  \mathcal{N}\right)  $ and
$q_{n-1}\left(  \mathcal{L}\right)  $, where $\mathcal{L}=\left\{
1,2,\ldots,2\ell-1\right\}  $. Set $n=2\ell$; then,
\begin{equation}
q_{2\ell}\left(  \mathcal{N}\right)  =\sum_{k=\ell+1}^{2\ell}\sum
_{\substack{I\subseteq\mathcal{N}\\\left\vert I\right\vert =k}}Q_{2\ell
,k}\left(  \mathcal{N},I\right) 
\end{equation}
with $\mathcal{N}=\ \left\{  1,2,\ldots,2\ell\right\}  $ and put
\begin{equation}
q_{2\ell-1}\left(  \mathcal{L}\right)  =\sum_{k=\ell}^{2\ell-1}\sum
_{\substack{I\subseteq\mathcal{L}\\\left\vert I\right\vert =k}}Q_{2\ell
-1,k}\left(  \mathcal{L},I\right)  ,
\end{equation}
notice that the number of terms is equal in both sums. If $k=2\ell$, then
\begin{equation}
Q_{2\ell,2\ell}\left(  \mathcal{N},\mathcal{N}\right)  =p_{2\ell}%
Q_{2\ell-1,2\ell-1}\left(  \mathcal{L},\mathcal{L}\right)  ,\label{2elTag}%
\end{equation}
otherwise,
\begin{equation}
Q_{2\ell,k}\left(  \mathcal{N},I\right)  =\left\{
\begin{array}
[c]{ccc}%
p_{2\ell}Q_{2\ell-1,k-1}\left(  \mathcal{L},I\backslash2\ell\right)   &
\text{if} & 2\ell\in I\\
\left(  1-p_{2\ell}\right)  Q_{2\ell-1,k}\left(  \mathcal{L},I\right)   &
\text{if } & 2\ell\notin I
\end{array}
\right.  ,
\end{equation}
hence for $k<2\ell$,
\begin{align}
\sum_{\substack{I\subseteq\mathcal{N}\\\left\vert I\right\vert =k}}Q_{2\ell
,k}\left(  \mathcal{N},I\right)    & =p_{2\ell}\sum_{\substack{I\subseteq
\mathcal{L}\\\left\vert I\right\vert =k-1}}Q_{2\ell-1,k-1}\left(
\mathcal{L},I\right)  \label{TobbiTag} \nonumber \\
& +\left(  1-p_{2\ell}\right)  \sum_{\substack{I\subseteq\mathcal{L}%
\\\left\vert I\right\vert =k}}Q_{2\ell-1,k}\left(  \mathcal{L},I\right)
.
\end{align}
We start summing up $q\left(  \mathcal{N},2\ell\right)  $ from $2\ell$; then,
using (\ref{2elTag}) and (\ref{TobbiTag}), we obtain for the first two terms
\begin{multline}
\sum_{k=2\ell-1}^{2\ell}\sum_{\substack{I\subseteq\mathcal{N}\\\left\vert
I\right\vert =2\ell-1}}Q_{2\ell,k}\left(  \mathcal{N},I\right)  =p_{2\ell
}Q_{2\ell-1,2\ell-1}\left(  \mathcal{L},\mathcal{L}\right)  \\
+p_{2\ell}\sum_{\substack{I\subseteq\mathcal{L}\\\left\vert I\right\vert
=2\ell-2}}Q_{2\ell-1,2\ell-2}\left(  \mathcal{L},I\right)  +\left(
1-p_{2\ell}\right)  Q_{2\ell-1,2\ell-1}\left(  \mathcal{L},\mathcal{L}\right)
\\
=Q_{2\ell-1,2\ell-1}\left(  \mathcal{L},\mathcal{L}\right)  +p_{2\ell}%
\sum_{\substack{I\subseteq\mathcal{L}\\\left\vert I\right\vert =2\ell
-2}}Q_{2\ell-1,2\ell-2}\left(  \mathcal{L},I\right)  .
\end{multline}
If we continue summing up one by one, then induction leads to%
\begin{multline}
q_{2\ell}\left(  \mathcal{N}\right)   =\sum_{k=\ell+1}^{2\ell-1}%
\sum_{\substack{I\subseteq\mathcal{L}\\\left\vert I\right\vert =k}%
}Q_{2\ell-1,k-1}\left(  \mathcal{L},I\right)  \\
 +p_{2\ell}\sum_{\substack{I\subseteq\mathcal{L}\\\left\vert I\right\vert
=\ell}}Q_{2\ell-1,\ell}\left(  \mathcal{L},I\right)
 <q_{2\ell-1}\left(  \mathcal{L}\right)  ,
\end{multline}
since $p_{2\ell}<1$.
\end{proof}

\section{Proof for Proposition \ref{prop:grfo}.}

\begin{proof}
\label{Proof_prop:grfo}
We prove the statement with an example describing the worst-case scenario
for the greedy selection strategy.
Let ${\cal{D}}
=\{D_{1}=(p_{1},t_{1}),D_{2}=(p_{2},t_{2}), \dots
, D_{n}=(p_{n},t_{n})\}$ be the pool, where the index set is denoted by ${\cal{I}}_n=\{1,2,\ldots,n\}$. Let us suppose that $\sum\limits_{i=1}^{n}t_{i}\leq T$,
that is, the time constraint should not be of concern. Let $%
p_{1}=1/2+\varepsilon $, where $0<\varepsilon \leq 1/2$, and $p_{2}=p_{3}=\dots
=p_{n}=1/2+\alpha $ with $0<\alpha <\varepsilon $, where the proper
selection of $\alpha $ will be given below.

The greedy strategy will move $D_{1}$ to $S$ as the most accurate item in its first step. Next, we
try to extend $S$ by adding more members. Since we require odd members, we
try to add 2 items in every selection step. Since all the remaining $n-1$
features have the same behavior, we can check whether $S$ should be extended
via comparing the performance of $S_{1}=\{D_{1}\}$ and $S_{
3}=\{D_{1},D_{2},D_{3}\}$. For the performance of the ensemble $S_{1}$, we
trivially have $q_1({\cal{I}}_1)=p_{1}=1/2+\varepsilon $, where ${\cal{I}}_1=\{1\}$, while for $S_{
3}$ we can
apply (\ref{accuracy}) for the 3-member ensemble, with ${\cal{I}}_3=\{1,2,3\}$ to calculate $q_3({\cal{I}}_{3})$:
\begin{eqnarray}
q_3({\cal{I}}_{3})
&=&p_{1}p_{2}(1-p_{3})+p_{2}p_{3}(1-p_{1})+p_{1}p_{3}(1-p_{2}) \nonumber
\\ 
&+&p_{1}p_{2}p_{3}=\frac{1}{2}+\frac{\varepsilon }{2}+\alpha -2\alpha ^{2}\varepsilon 
\end{eqnarray}%
after the appropriate substitutions and simplifications. Now, if we adjust $%
\alpha $ to have $q_1({\cal{I}}_1)=q_3({\cal{I}}_{3})$, then via solving the equation 
\begin{equation}
\frac{1}{2}+\varepsilon =\frac{1}{2}+\frac{\varepsilon }{2}+\alpha -2\alpha
^{2}\varepsilon 
\end{equation}%
we obtain 
\begin{equation}
\alpha =\frac{1-\sqrt{1-4\varepsilon ^{2}}}{4\varepsilon }.  \label{eq:greedy1}
\end{equation}%
That is, with a selection of $\alpha $ given in (\ref{eq:greedy1}), the
ensemble $S_1=\{D_{1}\}$ is not going to be extended since it does not lead
to improvement. Thus, the strategy stops after the first step with an
ensemble accuracy $1/2+\varepsilon$.

On the other hand, with a sufficiently large $n$, a very accurate ensemble
could be achieved. More precisely, it can be easily seen that $q_n({\cal{I}}_{n})$ is strictly monotonically increasing 
with 
\begin{equation}
\lim_{n\rightarrow \infty } q_n({\cal{I}}_{n})=1.
\end{equation}

Now, by letting $\varepsilon \rightarrow 0$, we can see that for the ensemble
accuracy found with this strategy 
\begin{equation}
\lim_{\varepsilon \rightarrow 0}q_1(S_{1})=1/2,
\end{equation}%
while an ensemble of $\lim\limits_{n\rightarrow \infty }q_n({\cal{I}}_{n})=1$ could
also be found. Hence, the proposition follows.
\end{proof}

\section{Proof for Proposition \ref{prop:grba}.}

\begin{proof} \label{Proof_prop:grba} 
We prove the statement with a similar example to that given in the proof of
Proposition \ref{prop:grfo} in Appendix \ref{Proof_prop:grfo} to describe the worst case scenario.
Let $
{\cal{D}}
=\{D_{1}=(p_{1},t_{1}),D_{2}=(p_{2},t_{2}),\dots
,D_{n}=(p_{n},t_{n})\}$ be the pool and $T$
be the time constraint. Put $p_{1}=1/2+\varepsilon $, where $0<\varepsilon \leq 1/2
$, $t_{1}=T$, and $p_{2}=p_{3}=\dots =p_{m}=1/2+\alpha $, $t_{2}=t_{3}=\dots
=t_{n}=\displaystyle {T}/({n-1})$ with $0<\alpha <\varepsilon $.
If $\alpha $ is properly selected, then $%
q_1({\cal{I}}_{1})=p_{1}<q_{n-1}({\cal{I}}_{n}\setminus {\cal{I}}_{1})$. However, because of the time constraint,
we must remove elements during the selection procedure, since initially $%
\sum\limits_{i=1}^{n}t_{i}=2T>T$. For this requirement, the greedy approach
in the first step will remove any two elements from $D_{2},\dots ,D_{n}$
by decreasing the time with $\displaystyle {2T}/({n-1})$. This selection will go on
until only $D_{1}$ remains in the ensemble. With a proper selection of $\alpha 
$, we have $\lim\limits_{n\rightarrow \infty }q_{n-1}({\cal{I}}_{n}\setminus {\cal{I}}_{1})=1$ and by letting $%
\varepsilon \rightarrow 0$, the proposition follows.
\end{proof}

\section{Proof for Proposition \ref{prop:grbau}.}
\label{proof_prop:grbau}

\begin{proof}
Similar to the proof of Proposition \ref{prop:grba}, we provide an example for the worst case scenario.
Let $D_{1}=(1,T)$, and $D_{2}=D_{3}=D_{4}=(1/2+\varepsilon ,T/3)$ with $%
0<\varepsilon <1/2$. Now, since 
\begin{equation}
u_{1}=\frac{1}{T}<\frac{3/2+3\varepsilon }{T}=u_{2}=u_{3}=u_{4},
\end{equation}%
the backward strategy will remove the less useful component $D_{1}$ first to
maintain the time constraint and will keep the remaining ensemble $\{D_2,D_3,D_4\}$
as the most accurate one, which also fits the time constraint with $\sum\limits_{i=2}^{4}t_i=T$.
By letting $\varepsilon\to 0$, we have $\lim_{\varepsilon\to 0} q_3({\cal{I}}_{4}\setminus {\cal{I}}_{1})
=1/2$. Moreover, notice that the most accurate ensemble would have been $\{D_1\}$ with $q_1({\cal{I}}_{1})=1$ by meeting the time constraint, as well. Thus,
the statement follows.
\end{proof}

\section{Proof for Lemma \ref{Lemma_Eq_Varq}}

\begin{proof}
\label{Proof_Lemma_Eq_Varq}The first part of the lemma corresponds to Theorem 1 in \cite{numbl}. For the rest, let us denote the product of probabilities by
\begin{equation}
\Pi\left(  I\right)  =%
{\displaystyle\prod\limits_{i\in I}}
p_{i}%
{\displaystyle\prod\limits_{j\in {\cal{N}}\backslash I}}
\left(  1-p_{j}\right)  ,
\end{equation}
for simplifying the treatment below. The formula (\ref{sz2}) follows from
expressing the variance in terms of covariance
\begin{align}
\operatorname*{Var}\left(  q_{\ell}\right)   =\sum_{k,j=k_{ \ell}}^{  \ell}\sum_{\substack{I,J\subseteq {\cal{N}}\\\left\vert
I\right\vert =k,\left\vert J\right\vert =j}}\operatorname*{Cov}\left(
\Pi\left(  I\right)  ,\Pi\left(  J\right)  \right)  .
\end{align}
Now, we rewrite this expression into a more appropriate form. First, the notation is introduced, where $I_{T}^{k}\ $and $I_{F}^{k}$ for a partition of indices
${\cal{N}}=\{1,\ldots,\ell\}$, such that ${\cal{N}}=I_{T}^{k}\cup I_{F}^{k}$ where $I_{T}^{k}$
denotes indices of those members voting true with accuracy $p$. Similarly,
$I_{F}^{k}$ contains indices of false votes. Observe $I_{F}^{k}={\cal{N}}\backslash
I_{T}^{k}$. We have $\left\vert I_{T}^{k}\right\vert =k$ and $\left\vert
I_{F}^{k}\right\vert =\ell-k$. In the case of two partitions $I_{T}^{k}\cup I_{F}%
^{k}$ and $J_{T}^{j}\cup J_{F}^{j}$, let the number of the common
elements of $I_{T}^{k}$ and $J_{T}^{j}$ be $\left\vert I_{T}^{k}\cap
J_{T}^{j}\right\vert =n_{k.j}$; similarly, $\left\vert I_{F}^{k}\cap
J_{F}^{j}\right\vert =m_{k.j}$. 
According to this setup%
\begin{align}
\operatorname*{Var}\left(  q_{\ell}\right) =\sum_{k,j=k_{\ell}}^{ \ell}\sum_{I_{T}^{k},J_{T}^{j}}\operatorname*{Cov}%
\left(  \Pi\left(  I_{T}^{k}\right)  ,\Pi\left(  J_{T}^{k}\right)  \right).
\end{align}
Observe $I_{F}^{k}={\cal{N}}\backslash I_{T}^{k}$ when we apply the notation for the
product. Now, we consider the covariance
\begin{multline}
  \operatorname*{Cov}\left(  \Pi\left(  I_{T}^{k}\right)  ,\Pi\left(
J_{T}^{k}\right)  \right)  \\
 =\operatorname*{E}\Pi\left(  I_{T}^{k}\right)  \Pi\left(  J_{T}^{k}\right)
-\operatorname*{E}\Pi\left(  I_{T}^{k}\right)  \operatorname*{E}\Pi\left(
J_{T}^{k}\right)  \\
  =\operatorname*{E}\Pi\left(  I_{T}^{k}\right)  \Pi\left(  J_{T}^{k}\right)
-\mu^{k+j}\left(  1-\mu\right)  ^{2\ell-k-j}.
\end{multline}
The first term contains three types of products:
\begin{align}
\operatorname*{E}p^{2} &  =s_{T}=\sigma^{2}_{p}+\mu^{2}_{p},\\
\operatorname*{E}\left(  1-p\right)  ^{2} &  =s_{F}=\sigma^{2}_{p}+\left(
1-\mu_{p}\right)  ^{2},\\
\operatorname*{E}p\left(  1-p\right)   &  =s_{TF}=\mu_{p}\left(  1-\mu_{p}\right)
-\sigma^{2}_{p}.
\end{align}
The 
pool constitutes independent variables;
therefore,
\begin{align} 
 \operatorname*{Var}&\left(  q_{\ell}\right)
=\sum_{k,j=k_{\ell}}^{ \ell}%
\sum_{I^{k},J^{j}}\left(  \sigma^{2}_{p}+\mu^{2}_{p}\right)  ^{n_{k.j}}\left(
\sigma^{2}_{p}+\left(  1-\mu_{p}\right)  ^{2}\right)  ^{m_{k.j}} \notag
\\ 
& \times\left(  \mu_{p}\left(
1-\mu_{p}\right)  -\sigma^{2}_{p}\right)  ^{   \ell- n_{k.j}-m_{k.j}} -\left(  Eq_{\ell}\right)  ^{2}. \label{Vari_qn1}
\end{align} 
since the sum of the second term gives the $\left(  Eq_{ \ell}\right)  ^{2}$,
indeed
\begin{multline}
\sum_{k,j=k_{\ell           }}^{\ell}\binom{\ell}{k}\binom{\ell}{j}\mu^{k+j}_{p}\left(
1-\mu_{p}\right)  ^{2\ell-k-j} \\   =\left(  \sum_{k=k_{\ell}}^{\ell}\binom{\ell}{k}\mu_{p}
^{k}\left(  1-\mu_{p}\right)  ^{\ell-k}\right)  ^{2} 
 =\left(  Eq_{\ell}\right)  ^{2}.
\end{multline}
We simplify (\ref{Vari_qn1}), collecting similar terms and obtain (\ref{sz2}).
Before we prove the limit (\ref{Var_Limit}), let us observe
\begin{align}
s_{T}+s_{TF}  &  =\mu_p,\\
s_{F}+s_{TF}  &  =1-\mu_p,\\
s_{T}+s_{F}+2s_{TF}  &  =1.
\end{align}
i.e., the set $\left\{  s_{T},s_{TF},s_{F},s_{TF}\right\}  $ constitutes a
probability distribution for $s_{TF}>0$; in other words, $\mu^{2}_{p}%
+\sigma^{2}_{p}<\mu_{p}$. If it is so, we rewrite (\ref{sz2}) in the form of
a multinomial distribution. The coefficients in (\ref{sz2}) are actually
multinomial coefficients. 
The rest of the proof is based on the approximation of the binomial distribution by the normal distribution. It is not complicated but slightly lengthy; we  make it  available to the interested readers on request.
\end{proof}

\section{Proof for Lemma \ref{LemmaTime}}

\begin{proof}
\label{Proof_LemmaTime} We show only the first statement; the rest of the lemma is well known.  If $\lambda\in\left(  0,1\right)  $ is distributed as
\emph{Beta} $\left(  \alpha_p,\beta_p\right)  $, then $1-\lambda$ is distributed as \emph{beta}
$\left(  \beta_p,\alpha_p\right)  $.  The expected value of time is calculated in two steps; first, we take the conditional expectation, namely,  
\begin{equation}
\begin{split}
E\tau =EE\left(  \left.  \tau\right\vert \lambda\right) 
=\int\limits_{0}^{1}%
\int\limits_{0}^{\infty}t\lambda\exp\left(  -\lambda t\right)dtb\left(
\lambda;\beta_p,\alpha_p\right)d\lambda\\
=\frac{\Gamma\left(  \beta_p-1\right)  }{\Gamma\left(  \alpha_p+\beta_p-1\right)
}\frac{\Gamma\left(  \alpha_p+\beta_p\right)  }{\Gamma\left(  \beta_p\right)  }%
  =1+\frac{\alpha_p}{\beta_p-1},
 \end{split}
\end{equation}
where we assumed that $1<\beta_p<\alpha_p$. 
Suppose  $2<\beta_p<\alpha_p$ to calculate the variance in a similar manner
\begin{equation}
\begin{split}
Var\left(  \tau\right)   &  =EE\left(  \left.  \left(  \tau-E\left(  \left.
\tau\right\vert \lambda\right)  \right)  ^{2}\right\vert \lambda\right)  \\
& =\int_{0}^{1}\frac{1}{\lambda^{2}}b\left(
\lambda;\beta_p,\alpha_p\right)  d\lambda
=1+\frac{\alpha_p}{\beta_p-2}.
\end{split}
\end{equation}

\end{proof}

\end{document}